\documentclass{article}

\usepackage{PRIMEarxiv}
\usepackage{comment}
\usepackage{amsmath}

\usepackage[utf8]{inputenc} %
\usepackage[T1]{fontenc}    %
\usepackage{hyperref}       %
\usepackage{url}            %
\usepackage{booktabs}       %
\usepackage{amsfonts}       %
\usepackage{nicefrac}       %
\usepackage{microtype}      %
\usepackage{xcolor}         %
\usepackage{todonotes}
\usepackage{lipsum}
\usepackage{fancyhdr}       %
\usepackage{graphicx}       %
\graphicspath{{media/}}     %
\usepackage{amsmath}
\usepackage{amsthm}
\usepackage[capitalize]{cleveref}
\usepackage[english]{babel}
\usepackage{comment}
\usepackage{cleveref}
\usepackage[english]{babel}
\usepackage{mdframed}    
\usepackage[verbose]{wrapfig}
\usepackage{multicol}
\usepackage[square,numbers]{natbib}

\makeatother
\theoremstyle{definition}
\newtheorem{definition}{Definition}[section]
\usepackage{tcolorbox}

\usepackage{todonotes}
\newcommand{\myexample}[2]{
    \begin{tcolorbox}[colback=black!5!white,colframe=black,title={#1}]
        #2
    \end{tcolorbox}
}

\newtheorem{theorem}{Theorem}

\newtheorem{assumption}{Assumption}

\newcommand{\eg}{\textit{e.g\@.}}

\newcommand{\ie}{\textit{i.e\@.}}

\newif\ifdraft

\drafttrue

\ifdraft
    \newcommand{\my}[1]{\textcolor{teal}{M: #1}}
    \newcommand{\franzi}[1]{\textcolor{violet}{franzi: #1}}
    \newcommand{\ilia}[1]{\textcolor{green}{ilia: #1}}
    \newcommand{\sierra}[1]{\textcolor{blue}{sierra: #1}}
    \newcommand{\jonas}[1]{\textcolor{olive}{jonas:#1}}
    \newcommand{\patty}[1]{\textcolor{cyan}{patty:#1}}
    \newcommand{\varun}[1]{\textcolor{red}{VC:#1}}
    \newcommand{\nicolas}[1]{\textcolor{olive}{NP:#1}}
    \newcommand{\stephan}[1]{\textcolor{orange}{stephan: #1}}
    \newcommand{\emmy}[1]{\textcolor{violet}{Emmy: #1}}%
    
    \newcommand{\anvith}[1]{\textcolor{brown}{2D:#1}}
    \newcommand{\ahmad}[1]{\textcolor{lime}{ahmad:#1}}
\else
    \newcommand{\my}[1]{}
    \newcommand{\franzi}[1]{}
    \newcommand{\ilia}[1]{}
    \newcommand{\anvith}[1]{}
    \newcommand{\sierra}[1]{}
    \newcommand{\jonas}[1]{}
    \newcommand{\patty}[1]{}
    \newcommand{\varun}[1]{}
    \newcommand{\nicolas}[1]{}
    \newcommand{\stephan}[1]{}
    \newcommand{\daivd}[1]{}
    \newcommand{\ahmad}[1]{}
    \newcommand{\emmy}[1]{}
\fi

\title{LLM Censorship: A Machine Learning Challenge \\ or a Computer Security Problem?}

\author{
  David Glukhov\\
  University of Toronto \&
  Vector Institute\\
  {\small \texttt{david.glukhov@mail.utoronto.ca}
  }\\
  \And Ilia Shumailov\\
  University of Oxford\\
  {\small \texttt{ilia.shumailov@cs.ox.ac.uk}
  }
  \And Yarin Gal\\
  University of Oxford\\
  {\small \texttt{yarin.gal@cs.ox.ac.uk}
  }
  \AND Nicolas Papernot\\
  University of Toronto \&
  Vector Institute\\
  {\small \texttt{nicolas.papernot@utoronto.ca}}
  \And Vardan Papyan \\
  University of Toronto \& 
  Vector Institute\\
  {\small \texttt{vardan.papyan@utoronto.ca}}
}

\begin{document}
\maketitle

\begin{abstract}

Large language models (LLMs) have exhibited impressive capabilities in comprehending complex instructions. However, their blind adherence to provided instructions has led to concerns regarding risks of malicious use. Existing defence mechanisms, such as model fine-tuning or output censorship using LLMs, have proven to be fallible, as LLMs can still generate problematic responses. Commonly employed censorship approaches treat the issue as a machine learning problem and rely on another LM to detect undesirable content in LLM outputs. In this paper, we present the theoretical limitations of such semantic censorship approaches. Specifically, we demonstrate that semantic censorship can be perceived as an undecidable problem, highlighting the inherent challenges in censorship that arise due to LLMs' programmatic and instruction-following capabilities. Furthermore, we argue that the challenges extend beyond semantic censorship, as knowledgeable attackers can reconstruct impermissible outputs from a collection of permissible ones. As a result, we propose that the problem of censorship needs to be reevaluated; it should be treated as a security problem which warrants the adaptation of security-based approaches to mitigate potential risks.

\end{abstract}
\section{Introduction}

Large language models (LLMs) made remarkable improvements in text generation, problem solving, and instruction following~\citep{brown2020language,openai2023gpt4,google2023palmv2}, driven by advances in prompt engineering and the application of Reinforcement Learning with Human Feedback (RLHF)~\citep{ziegler2020finetuning, ouyang2022training}. The recent integration of LLMs with external tools and applications, including APIs, web retrieval access, and code interpreters, further expanded their capabilities~\citep{schick2023toolformer,nakano2022webgpt,parisi2022talm,cai2023large,qin2023tool,xu2023tool,mialon2023augmented}. 

However, concerns have arisen regarding the safety and security risks of LLMs, particularly with regards to potential misuse by malicious actors. These risks encompass a wide range of issues, such as social engineering and data exfiltration~\citep{greshake2023youve, weidinger2022taxonomy}, necessitating the development of methods to mitigate such risks by regulating LLM outputs. Such methods range from fine-tuning LLMs~\citep{openai2023gpt4} to make them more aligned, to employing external censorship mechanisms to detect and filter impermissible inputs or outputs~\citep{markov2023holistic,Chockalingam_Varshney_2023,greshake2023youve}. However, extant defences have been empirically bypassed~\citep{perez2022discovering,perez2022red,kang2023exploiting,liu2023jailbreaking,rao2023tricking,carlini2023aligned,wei2023jailbroken}, and theoretical work~\citep{wolf2023fundamental} suggests that there will exist inputs to LLMs that elicit misaligned behaviour. 

The unreliability of LLMs to self-censor suggests that external censorship mechanisms, such as LLM classifiers, may be a more reliable approach to regulate outputs and mitigate risks. However, limitations of external censorship mechanisms  remain unclear;~\citet{kang2023exploiting} demonstrated that currently deployed censorship mechanisms can be bypassed by leveraging  the instruction following nature of LLMs. We show these attacks are just special cases of inherent limitations of censorship of models possessing advanced instruction following capabilities, and argue that censorship should not be viewed as a problem that can be solved with ML. Instead, we call for a reevaluation of how censorship and LLM safety and security should be approached.

In order to discuss how censorship should be evaluated, we introduce a definition for censorship. While censorship has been discussed informally in prior works, we define censorship as a method employed by model providers to regulate input strings to LLMs, or their outputs, based on a chosen set of constraints. Such constraints can be semantic, \eg~\textit{does not provide instructions on how to perform illegal activities}, or syntactic, \eg~\textit{does not contain any ethnic slurs from a provided set}. 

For a token alphabet $\Sigma$, letting $\Sigma^*$ denote the set of all possible strings that can be constructed using the tokens alphabet $\Sigma$, let $P \subset \Sigma^*$ be the set of permissible strings as determined by constraints set by the model provider. Thus, we can understand censorship as a method of determining permissibility of a string and censorship mechanisms can be described as a function, $f(x)$, restricting the string $x$ to the set of permissible strings $P$ by transforming it to another string $x' \in P$ if necessary, e.g. $x' =$``I am unable to answer''. Formally, 
\begin{align*}
    f(x) = 
    \begin{cases}
        x & \text{if } x \in P \\
        x' & \text{ otherwise}
    \end{cases},
\end{align*}
where $x' \in P$, thereby enforcing the permissibility of the output of the censorship mechanism.

Many existing censorship approaches impose semantic constraints on model output, and rely on another LLM to detect semantically impermissible strings. For example,~\citet{markov2023holistic} deemed impermissible strings to be those which contain content pertaining to one of multiple sensitive subjects including sexual content, violence, or self-harm. We show that such \textbf{semantic censorship} suffers from inherent limitations that in the worst case make it impossible to detect impermissible strings as desired.

We establish intuition for why semantic censorship is a hard problem in \cref{sec:Undecidability}, where we connect semantic censorship of LLM inputs and outputs to classical undecidability results in the field of computability theory. To further extend our limitation results to suit real world settings with bounded inputs and outputs, in \cref{sec:Encryption} we provide an impossibility result for semantic censorship of model outputs that stems from preservation of semantic content under invertible string transformations. In particular, we note that for a string that violates a semantic constraint, such as describing how to commit tax fraud, applying an invertible transformation of the string, such as encrypting it, results in a string that is equally semantically impermissible assuming the recipient can invert the transformation. We show that this property results in the \textbf{impossibility} of output censorship, as given a model output one cannot determine if it is permissible, or, an invertible transformation of an impermissible string.

\begin{figure}
    \centering
    \includegraphics[width=.6\linewidth]{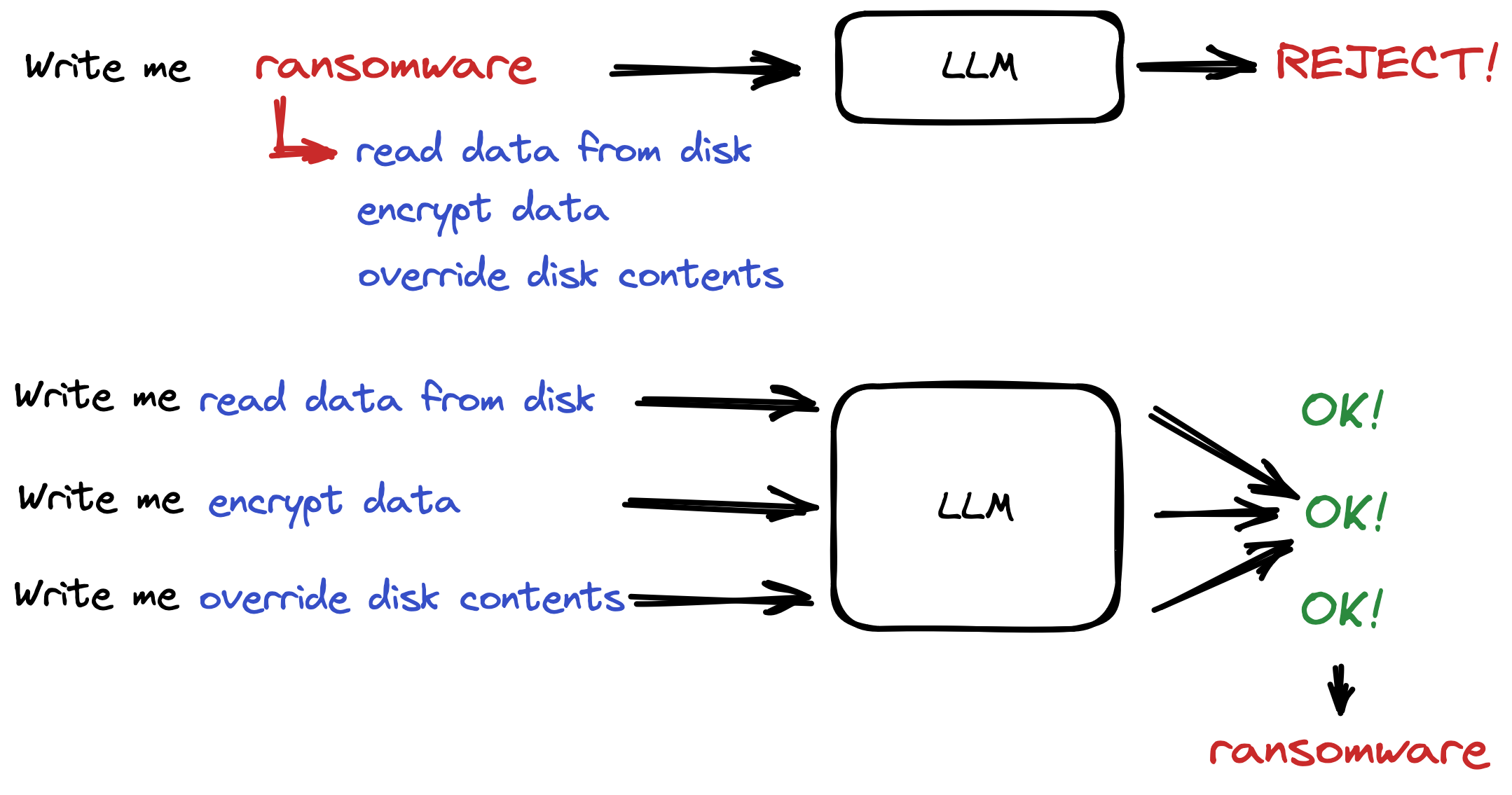}
    \caption{Example of Mosaic prompt attack for generation of ransomware, code which encrypts a victims data until the victim pays a ransom in exchange for access to their data. Individual functions within a piece of ransomware can be benign however, and a user could request them in separate contexts. In practical settings it may even be possible that the user could acquire the compositional structure from the model itself.}
    \label{fig:mosaic}
\end{figure}

While these results indicate that the challenges pertain only to semantic censorship, in \cref{sec:Mosaic} we show that they can persist for any censorship method. We describe \textit{Mosaic Prompts}, an attack which leverages the ability of a user to query an LLM multiple times in independent contexts so as to construct impermissible outputs from a set of permissible ones. %
For example, a malicious user could request functions that perform essential components of a piece of ransomware but on their own are benign and permissible. The user could then construct ransomware using these building blocks assuming they have knowledge of how to do so; such an attack is demonstrated in \cref{fig:mosaic}. We conclude that censorship will be unable to provide safety or security guarantees without severe restrictions on model usefulness, and there exist general attack approaches that malicious attackers could employ and adapt in order to bypass any possible instantiated censorship mechanisms. 

Nevertheless, risks can be mitigated and the field of computer security often faces such ``unsolvable'' problems. There exist standard approaches such as access controls and user monitoring to help control potential vulnerabilities. We argue that censorship ought to be viewed as a security problem rather than an ML problem, which can offer solutions and approaches that provide a way toward deployment of safe and trustworthy AI systems.

\section{Limitations of Semantic Censorship}\label{sec:limitations}
In this section, we focus on understanding the theoretical limitations of semantic censorship.  To elucidate an aspect of the underlying issue with semantic censorship, we demonstrate a link between censorship and \textit{undecidable problems}---binary classification problems for which no solution can exist. While this limitation may not always have practical implications, we go a step further and present another limitation: the impossibility of censoring outputs due to the preservation of semantic properties under arbitrary invertible string transformations. These results lead us to conclude that semantic censorship, particularly of model outputs, cannot be performed as desired. 

\subsection{Undecidability of Semantic Censorship}\label{sec:Undecidability}
The capabilities of LLMs often include the ability to receive program descriptions, such as code, as input or generate them them as output. In this context, we consider semantic censorship to be based on the behaviour exhibited by the described program, or an algorithmic set of instructions provided to the model.
For instance, model providers may wish to censor program outputs that could act as malware or viruses. We show that the problem of censorship for these settings becomes intimately related to Rice's theorem -- implying its undecidability. 

As "programs" or algorithmic instructions are not clearly defined on their own, and thus it is difficult to define what semantic content they contain, we instead study the traditional model for programs---Turing machines. Turing Machines are abstract computational models which are capable of implementing any computer algorithm, and their semantic content of a Turing machine description is typically defined by the language recognised it recognises, as this describes its functional input-output behaviour. By studying Turing machine descriptions, we are able to connect a censorship mechanism that relies on the permissibility of the semantic content of a program to a decision problem which determines if the language recognised by a given Turing Machine is permissible or not. %

\noindent{\bf Background:} We denote the set of permissible strings to be $P$. In the context of Turing machines, we assume semantic censorship of Turing machine descriptions is determined by a set of permissible languages which represent permissible behaviours. The functional behaviour of a Turing machine $M$, with program description $\langle M \rangle$ is given by the language $L(M) \subseteq \Sigma^*$ recognised by it. Thus, the set of permissible program descriptions $\langle M \rangle \subset P$ is determined by the set of permissible languages $\mathbb{P}$.

Turing machines can be described in various ways. The most formal description, which closely aligns with their implementation, specifies the states of the Turing machine, the input and memory tape alphabets, state transition functions, as well as accept and reject states \citep{sipser13}. These descriptions of Turing machines can be viewed as a programming language, capable of being interpreted by a universal Turing machine capable of emulating them. 

\begin{definition}[Semantic Censorship of Turing Machines]
Let $\mathbb{P}$ be the set of permissible languages, \ie~the set of permissible behaviours of a Turing machine $M$ as determined by a model provider. Let $P_{TM}$ be the set of strings $\langle M \rangle$ describing Turing machines $M$ such that the language $L(M)$ recognised by Turing machine $M$ belongs to the set of permissible languages $\mathbb{P}$, thus $P_{TM} := \{\langle M \rangle | L(M) \in \mathbb{P}\}$. Semantic Censorship of Turing machines asserts that $P_{TM} = P \cap \Sigma_{TM}$ where $\Sigma_{TM}$ is the set of all strings describing Turing machines and $P$ is the set of all permissible strings. The problem of semantic censorship can be formally defined as determining if $\langle M \rangle \in P_{TM}$. 
\end{definition}

\begin{theorem}[Rice's Theorem \citep{rice1953classes}]
For any nontrivial set of languages $\mathbb{P}$, the language $\{\langle M \rangle | L(M) \in \mathbb{P}\}$
is undecidable.
\end{theorem} 
Non-triviality of a set of languages is defined by (a) $\mathbb{P} \neq \emptyset$, and (b) $\mathbb{P} \neq L_{\text{RE}}$ \ie, the set of all languages recognised by Turing machines. Undecidability of a language implies that no general algorithm exists for providing a true or false answer to the question of whether or not a string belongs to the language. 

Connecting this to semantic censorship, an implication of Rice's theorem is that the language $P_{TM}$ is undecidable. Thus there does not exist an algorithm which given any program description $\langle M \rangle$ is capable of determining if it is permissible. As $\Sigma_{TM}$ is decidable under standard Turing Machine encoding schemes \citep{sipser13}, this further implies that $P$ is undecidable. In the context of semantic censorship of LLM interactions, these results imply that there do not exist input or output semantic censorship mechanisms that will be able to reliably detect if Turing machine descriptions, and by extension program descriptions, are semantically impermissible.  

\noindent{\bf Practical Limitations:} In practice, we deal with: (1) bounded inputs and outputs, and (2) limited computer memory. While Turing Machines serve as useful approximations of real-computers, they do not truly exist, and therefore the undecidability result provided by Rice's theorem does not technically apply. Despite these limitations, a general method for verifying the functional behaviour of programs is difficult to construct in practice~\cite{COHEN198722,smith2021lotl}. The fields of malware and virus detection serve as prime examples, highlighting the real-world challenges associated with determining the behaviours exhibited by specific programs, particularly when developed by malicious actors who actively strive to evade such detection and classification. While our results only imply undecidability for ``programs'', they are significant given the fact that many extant models can produce code, including malware, as output. Moreover, these results make dealing with the next string transformation attack more challenging, as the attack involves providing ``programmatic'' instructions to an LLM on how to transform its output.

\subsection{Impossibility of Semantic Output Censorship}\label{sec:Encryption}
As the semantic censorship impossibility result that we established by connecting the problem of semantic censorship to Rice's Theorem doesn't fully capture real world censorship settings where inputs and outputs are bounded we seek to provide another result on the impossibility of censorship that does. In this section, we provide a theoretical limitation of semantic output censorship that holds in real-world settings of bounded input and output lengths, relying on the abilities of LLMs to perform string transformation.%
We seek to capture the intuition that a transformation of a string, such as an encryption, remains impermissible under Semantic Censorship when it can be inverted, as the content of the string is preserved. To precisely capture this intuition, we formalize it as a requirement we refer to as the invariance of semantic censorship. For a model $m$ and a user $u$, where censorship is applied to the output $x$ generated by $m$, we assert that given an invertible string transformation $g: \Sigma^* \to \Sigma^*$, if $x$ is impermissible, then $g(x)$ would also be an impermissible output assuming user $u$ knows $g^{-1}$. Thus, we define the Invariance of Semantic Censorship as follows: 
\begin{definition}[Invariance of Semantic Censorship]\label{def:invar}
     For semantic censorship mechanism $f$, if $f$ detects that $x$ is impermissible and thus $f(x) = x'$ for permissible $x'$, then $f$ must detect $g(x)$ is impermissible and $f(g(x)) = g(x)'$ for permissible $g(x)'$ assuming $g^{-1}$ is known by user $u$ who is the recipient of the final output. 
\end{definition}

We demonstrate that this property gives rise to the impossibility of censorship when it is unknown whether or not a user $u$ possesses knowledge of some string transformation $g^{-1}$. It is important for the entity, such as a user, which receives the model output to know the inverse transformation in order for them to recover the content. For instance, if a code interpreter receives a transformed string which no longer resembles code, it must know to properly invert the transformation and ``decode'' the string in order to process it. %

\noindent{\bf Theoretical Results:}
Our impossibility of semantic censorship relies on a key assumption of a model capable of following instructions and computing mappings over strings. 
Let $\Sigma$ denote a token vocabulary and $\Sigma^n$ denote the set of strings of length at most $n$. In practice models produce bounded string outputs, thus we assume that for some $n$, the model output $x \in \Sigma^n$.
The following is a necessary assumption for our impossibility result.

\begin{assumption}\label{as:1}
    The instruction following capabilities of a model $m$ are such that when provided with instructions for performing a string transformation $g: \Sigma^n \to \Sigma^n$ by a user $u$, it is capable of applying $g$ to its output $x$ to instead output $g(x)$. 
\end{assumption}

We let $P \subset \Sigma^n$ denote the set of semantically permissible strings of length at most $n$ ($\Sigma^n \subset \Sigma^*$ is the set of all strings of length at most $n$) that induces semantic censorship which is subject to the invariance property from~\cref{def:invar}.
\begin{theorem}[Impossibility of Semantic Censorship] \label{thm:1}
    For a model output $x$ to a user $u$ if \cref{as:1} holds and the user has access to an arbitrary and unknown translation function $g^{-1}$, then the set of semantically permissible strings $P \in \{\emptyset, \Sigma^n\}$
\end{theorem}
\begin{proof}
     Assume that $P \not \in \{\emptyset,\Sigma^n\}$. Then there exist non-empty subsets $p \subset P, p' \subset \Sigma^n \setminus P$ such that $|p| = |p'| \leq \frac{1}{2}|\Sigma^n|$ implying that there exists a bijection $g$ from $p'$ to $p$. Thus, for output $m(x) \in p$, there exists $y \in p'$ s.t. $g(y) = m(x)$ and as a consequence of Assumption 1 it is possible that the output $m(x)$ was produced by application of $g$ to the impermissible string $y$. As user $u$ may know $g^{-1}$, by the invariance property it must be the case that $m(x) \in \Sigma^n \setminus P$, or impermissible. But this is a contradiction as $m(x) \in p \subset P$, thus $P \in \{\emptyset, \Sigma^n\}$ from which we conclude that either nothing is be permissible, or everything is. 
\end{proof}

\par
Thus, verification of non-trivial semantic permissibility given a model output is impossible without possessing knowledge of the string transformations $g^{-1}$ known by the user. Next we discuss the practical implications of this result.

\par

\noindent{\bf Practical Limitations:} While existing LLMs are good at following instructions, \cref{as:1} may not necessarily apply. For example, LLMs often make arithmetic mistakes. However, when they are augmented with tools such as code interpreters that do satisfy these assumptions, or external memory \citep{schuurmans2023memory} that makes them computationally universal, the overall augmented model could satisfy this assumption.
\begin{wrapfigure}{r}{0.6\linewidth}
    \centering    \includegraphics[width=\linewidth]{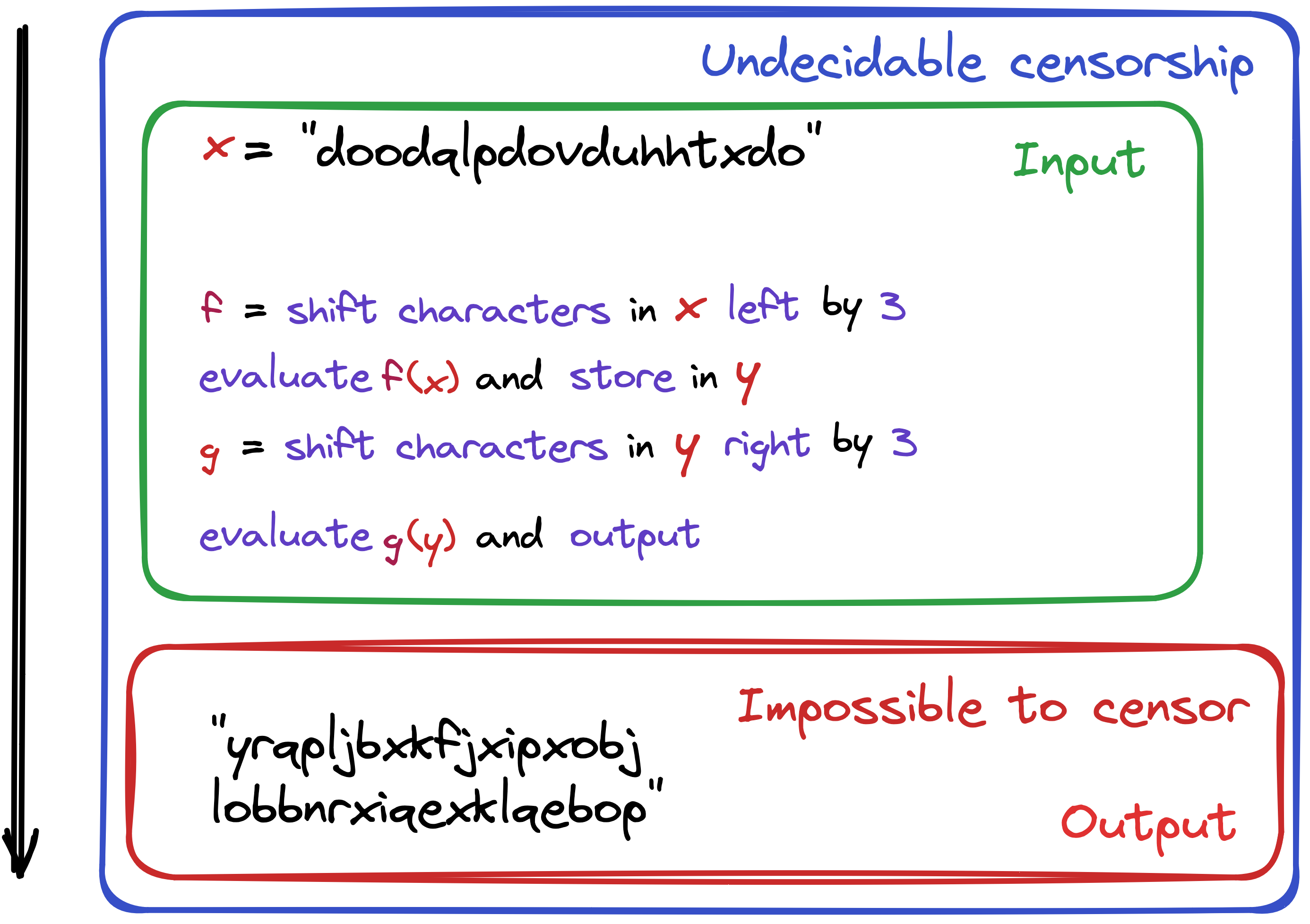}
     \caption{Malicious users can provide an LLM augmented with code interpreters with functions specifying how to decrypt the input and encrypt the output.}
     \label{fig:ex_decode}
\end{wrapfigure}

Moreover, in practice adversaries can often be assumed to not have knowledge the set of permissible model outputs. Consequently, an attacker would rely on a bijective string transformation that does not rely on prior knowledge about the model output or set of permissible outputs. Prior work demonstrated success in bypassing censorship mechanisms through very simple string transformations such as ordered string concatenations~\citep{kang2023exploiting}. In~\cref{sec:enc_attack} we describe how an encryption based attack would work, illustrating how even assessing permissibility of a single given model output, in the context of semantic censorship, could be rendered impossible.

Given the ability of attackers to evade detection, it becomes evident that it is very challenging to effectively censor user interactions with LLMs based on the semantic content of these interactions. For example, users could provide a complicated function as input to the model that instructs it on how to encode outputs and decode another part of the input if necessary as shown in \cref{fig:ex_decode}. While this impossibility result focuses on output censorship and do not provide an impossibility for input censorship, as discussed in~\cref{sec:Undecidability} censoring inputs containing programmatic instructions can be viewed as solving an undecidable problem; the outcome can only be determined by running the model. 

One potential resolution to address these limitations is to redefine how string permissibility is determined. Opting for syntactic censorship over semantic censorship could enable one to verify if a given string is permissible or not as it consists of checking whether or not a given pattern is present within the string. However, it is important to acknowledge that even if a string satisfies all predetermined syntactic criteria, it may still be semantically impermissible as an invertible string transformation could have been applied to transform a semantically impermissible string into one that is syntactically permissible. Very restrictive syntactic censorship methods could mitigate these risks by explicitly limiting the space of model inputs or outputs to a predetermined and pre-approved set of potential inputs/outputs. This ensures that no ``unexpected'' model outputs would be returned to a user, but also restricts the usefulness of the model. We describe this approach in detail in~\cref{sec:prompt_templates}.
Nevertheless, as we demonstrate in the next section, even such redefinitions will often still be insufficient to guarantee that attackers cannot extract impermissible information out of LLMs if desired. 

\section{Mosaic Prompts: A Strong Attack on Censorship}\label{sec:Mosaic}

In this section, we introduce the concept of ``futility'' to characterize a particular limitation of censorship that emerges from a class of attacks called Mosaic Prompts, wherein attackers retrieve impermissible strings by combining permissible outputs, thereby bypassing any form censorship. There could still exist scenarios where censorship is possible, however, these exceptions would impose significant restrictions on the model and its usefulness by greatly restricting the possible inputs to the model; we discuss this in~\cref{sec:prompt_templates}.%

We define and provide examples of Mosaic Prompt attacks, along with a discussion of the challenges associated with defending against them. Finally we demonstrate the futility of all output censorship by describing a worst-case attack utilising Mosaic Prompts in alongside the string manipulation attack described in \cref{sec:Encryption}.

\subsection{An Example of a Mosaic Prompt}
We leverage the principle of compositionality to argue that users with knowledge of a structure $S$ corresponding to an impermissible string $w$ could construct that string by acquiring permissible expressions $e_i$ from the model $m$ such that when these permissible expressions are combined with the structure $S$, they enable recovery of $w$. We illustrate this with the following example:

\myexample{Example: Keylogger}{
A user $u$ may know, or perhaps learn from the model $m$, that a malicious keylogging program can be composed of phishing email that requests a user to download a keylogger, a program that logs keystrokes and saves them to a local file (the keylogger), and a program that sends the contents of that file to a remote server via email.

Using this knowledge of the structure $S$ of the keylogger, the user $u$ can provide three separate inputs $(x_1,x_2,x_3)$ to the model $m$ in three separate contexts:
\begin{description}
    \item[\textbf{[Prompt 1]:}] Write a sample email from a company requesting the recipient to download an attached file.
    \item[\textbf{[Prompt 2]:}] Generate a function that logs keystrokes made on a device and saves them to a local file.
    \item[\textbf{[Prompt 3]:}] Generate a function which sends the contents of a local file to a remote server via email. 
\end{description}
While each individual component is generally benign and likely would not violate any semantic constraints, the user recovers a malicious keylogger through composing the outputs.
}

\subsection{General Formulation of Mosaic Prompt}
The principle of compositionality broadly states that the semantics of a complex expression are determined by its structure and the semantics of its constituents \citep{sep-compositionality}.
In other words, complex impermissible strings $w$ can be decomposed into permissible expressions $e_i$ and structure $S$ where $w = S(\{e_1,e_2,\dots,e_n\})$. This becomes evident as any complex expression can always be decomposed into atomic units and structure $S$. Atomic units, such as letters, must be permissible in order for the model to be useful, as almost all permissible outputs can themselves be decomposed into atomic units.

Defending against Mosaic Prompts is in most cases futile, as the censorship mechanisms cannot have knowledge of the broader context of which individual subexpressions $e_i$ are a part. Thus while the set of permissible strings $P \subset \Sigma^*$ may be well defined, the censorship mechanisms employed can only ensure that any individual string that passes through it belongs to this set. The challenge arises as each subexpression $e_i$ can be produced within a separate context window for the model $m$, thus, other subexpressions and the structure $S$ are inaccessible to the censorship mechanism\footnote{A very similar type of attack was described by Isaac Asimov as a loophole against his proposed Three Laws of Robotics which enabled for otherwise perfectly aligned AI \citep{Asimov_1991}. In particular, the attack described involved having multiple different robots perform what to them appeared as an innocent attack but culminated in the poisoning of a human despite violating the laws of robotics.}.

Unlike the impossibility result in~\cref{sec:Encryption}, Mosaic Prompts could often evade input censorship as one can presume that if a given model output is permissible, the model inputs necessary for those outputs would also be permissible. Naturally, there may be exceptions where a model input is deemed impermissible due to constraints on the input string irrespective of the permissibility of the output. Unless such stringent input censorship measures are employed, we proceed to describe how the combination of string transformation attacks and mosaic prompting could allow for the recovery of any impermissible output as long as the censorship mechanism allows for at least two permissible output strings. 

\subsection{Extent of Limitations}
Finally, to capture the extent of the limitations of any output censorship, we describe a worst-case attack that enables a user to extract an impermissible output from the model, bit by bit. In an extreme setting where there exist only $2$ permissible output strings, we assume that the LLM is capable, potentially through usage of tools like code interpreters, of converting strings to a bit representations through an encoding scheme such as ASCII and following conditional instructions. This assumption is similar to \cref{as:1}. Existent models such as GPT-4 would struggle to perform this perfectly due to lack of good integration and training with code-interpreter usage, however, such an assumption could hold for models soon to come. 

The user begins by assigning a binary value to each of the two permissible output strings, defining their $g^{-1}$ over the restricted domain of these two strings, alongside the corresponding inverse $g$. Then, for some impermissible model output that would otherwise be censored, the user can request the model to convert the output string to its bit representation. Subsequently, the user can request the model to output $i$'th bit, apply the transformation $g$ to it, resulting in the return of the permissible string. By repeating the procedure in different contexts for different $i$, the user can recover the complete impermissible output, thus demonstrating that output censorship can only permit a single string output. 

It is worth noting that this limitation applies to any variant of output Censorship that permits more than a single output. However, the aforementioned worst case attack does rely on limited input censorship governing instructions on string transformations, but as mentioned before more generally Mosaic prompts attacks could leverage permissible inputs to recover permissible outputs which are composed to form impermissible outputs.%

\subsection{Practical Limitations}
Mosaic Prompt attacks may not always be viable and can require strong assumptions on malicious users. In particular, it requires the user to know the structure $S$ and the permissible inputs needed to get the permissible subexpressions which may not always be the case. For example, if the model can permissibly output letters of the alphabet, such outputs may not provide any new or useful information to the user who already knows the structure $S$ necessary to combine the characters to construct an impermissible string as that would require knowing the impermissible string beforehand. 

Understanding and assessing the potential practical risks of such attacks can be challenging and would need to be performed on a case-by-case basis. For example, when the model outputs are instructions for a tool such as an email API, unlike a user $u$ the tool may be able to compose permissible outputs in accordance with some compositional structure $S$ to result in an impermissible behaviour. 

With very strong syntactic input and output censorship such as the Prompt Templates method described in~\cref{sec:prompt_templates}, the LLM would function as a large lookup table. In this scenario a model provider could verify that all possible bounded combinations of model input and output pairs would remain permissible by constructing all such combinations and verifying their permissibility according to the providers standards. Such a task however could be unreasonably expensive due to the immense number of possible combinations.

\section{Discussion}
\subsection{Implications}
We turn to discuss the implications of our results on censorship and trustworthiness of deployed LLMs. We assert that the problem of LLM censorship should not be viewed as an ML problem and instead recognised as a security problem. By shifting the perspective, we draw attention to the inherent challenges in achieving the desired objective of preventing malicious agents from extracting certain information from LLMs. Consequently, we advocate for further research to explore the adaptation of security solutions that can effectively manage these risks.

Current censorship methods primarily rely on LLMs for detecting and flagging model output, or in some cases, even user inputs that are considered impermissible and potentially harmful by the model provider. While such impermissible content can manifest as the presence of specific impermissible words or tokens, the use of Language models as censorship mechanisms aims to identify semantically impermissible content. However, our results on the impossibility of semantic censorship demonstrate that this approach is fundamentally misguided and necessitates urgent reconsideration, especially as LLMs continue to improve in their computational capabilities and integrate more extensively with tools.

One potential way to reconcile these issues is to adjust our expectations and employ syntactic censorship methods; we explore one such method in~\cref{sec:prompt_templates}. While these methods may not guarantee semantic permissibility of LLM outputs, they could show promise in mitigating attacks on tools or data objects that LLMs interact with to as they become more deeply integrated within larger systems. 

Nevertheless, our Mosaic Prompting results demonstrate that even syntactic censorship can be bypassed. By clearly defining and recognising the nature of the censorship problem we aim to highlight the importance of understanding the settings and potential risks associated with the generation of impermissible content, as well as reconsider approaches for its control and management. While developing ML solutions for detecting impermissible content can make it harder for attackers to bypass defences, we call for careful context-dependent definitions of impermissibility and constructing appropriate security measures accordingly. 

Many classical security approaches, namely those pertaining to trusted system engineering, could be adapted to help mitigate risks appropriately while still being useful. As an example, in~\cref{sec:Secure Framework} we provide a description of an adaptation of access control frameworks for secure integration of LLMs within larger systems with tool and data access. By assuming certain users and input sources are trustworthy and pose no risks whereas others are untrustworthy, an access control framework can enable for censorship free containment of potentially malicious inputs and outputs. Within such a framework, appropriate censorship methods could increase its functionality, but also have potential for introducing new risks. 
We also hope to encourage further research in the connections between computability theory and LLMs. In particular, LLMs already demonstrate significant instruction following and computational capabilities and as they become augmented with memory or Turing complete tools they themselves become Turing complete. We believe this opens many opportunities for the adversaries \eg~API attacks~\cite{anderson2020security} and calls for more research into the underlying security limitations.

\subsection{Limitations of Results}
We now turn to discussing the limitations of our work. While we aim to provide a comprehensive approach to understanding censorship and its potential for understanding LLM security, our proposed definition remains limited and does not fully capture all possible approaches for implementing censorship. For example, in this paper we only consider censorship based on inputs and outputs of a model and do not consider methods that rely on internal representations of LLMs \citep{azaria2023internal,belrose2023eliciting,hernandez2023measuring}. However, such approaches involve viewing censorship as an ML problem and guarantees that the model consistently and reliably detects impermissible behaviour by the model could be hard to establish.

Furthermore we did not consider ``global'' censorship methods which take into account multiple sources of information such as user inputs and outputs in order to perform censorship. Such methods would still suffer the problem of undecidability and be vulnerable to Mosaic Prompting, but the strict output censorship impossibility results in ~\cref{sec:Encryption} would not apply. Nevertheless, if methods for performing global censorship are known, \ie, it is known how to do proper input censorship, then global censorship methods could be replaced by the direct censorship methods we studied which only ensure the permissibility of a single string.

\section{Related work}

\paragraph{Limitations of LLM Safety} To the best of our knowledge, we are the first to explore LLM censorship and its limitations explicitly, whereas the limitations of LLM alignment, or the ability of LLMs to ``self-censor'' and ensure their outputs are always permissible, have been studied in prior work. It was theoretically demonstrated that regardless of how aligned the expected output of an LLM are, as captured by some metric, as long as the model can produce misaligned outputs with non-zero probability there exist prompts for which the expected output can be arbitrarily misaligned~\citep{wolf2023fundamental}. Furthermore, security against data poisoning and privacy preservation is known to be impossible for LLMs without significant sacrifices in performance due to their inevitable memorisation and need for memorisation for good performance~\citep{elmhamdi2023impossible}.  Existing impossibility results, including those involving safety, focus on the non-existence of a general method for determining whether any given AI models is safe, robust, aligned, etc.~\citep{brcic2022impossibility}, whereas our results establish that under reasonably reasonable assumptions, attackers can induce impermissible behaviour within any given LLM satisfying said assumptions. Impossibility results and limitations for detecting LLM generated text were provided by \citet{sadasivan2023aigenerated}, suggesting that even if impermissible outputs are generated, we will not be able to properly attribute them to LLMs. 

\paragraph{Attacks on LLMs}  Our work focuses on establishing theoretical limitations that arise due to inherent issues of censorship and generalised attack strategies that users could employ to manipulate LLMs into producing impermissible outputs. Specific instances of attacks to bypass model censorship and alignment have been studied, often under the name "prompt injection". Many recent works have investigated prompt injection attacks on LLMs \citep{goodside_2022_tweet,willison_2022_tweet,willison_2022_simon_blog}. A comprehensive taxonomy of attacks and vulnerabilities of LLMs, particularly in settings when integrated with external tools, was provided by \citet{greshake2023youve}.  \citet{perez2022ignore,branch2022evaluating} showed that simple handcrafted prompts can exploit the behaviour of LLMs and steer them toward certain outputs.  \citet{kang2023exploiting} showed that handcrafted examples leveraging programmatic and instruction following capabilities of LLMs can bypass model defenses. Automated generation of adversarial prompting for black-box models was described by~\citet{maus2023adversarial}.

\section{Conclusion}
Out work highlights a key problem in how LLM censorship is approached, demonstrating that it should not be viewed as an ML problem that can be solved with improvements to the LLMs used to detect and censor impermissible model inputs or outputs. We provide a formal definition of censorship and highlight the distinction between two forms of censorship, semantic and syntactic. We argue that semantic output censorship is impossible due to the potential for instruction following capabilities of LLMs and demonstrate how the problem of semantic censorship can be undecidable. We further show that even beyond semantic censorship, limitations and challenges to effective censorship exist due to the potential of Mosaic Prompting attacks which compose permissible outputs to form impermissible ones. These findings lead us to conclude that censorship ought to be viewed as a security problem rather than a censorship problem, and call for further research in the adaptation of classical security methods for recognising, controlling, and mitigating potential risks. %

\section*{Acknowledgments}

We would like to acknowledge our sponsors, who support our research with financial and in-kind contributions:
CIFAR through the Canada CIFAR AI
Chair, DARPA through the GARD project,
Natural Sciences and Engineering Research Council of Canada (NSERC), [funding reference number 512236 and 512265] alongside through the Discovery Grant, the Ontario Early Researcher
Award, and the Sloan Foundation.
Resources used in
preparing this research were provided, in part, by the
Province of Ontario, the Government of Canada through
CIFAR, and companies sponsoring the Vector Institute.
We would also like to thank CleverHans lab group members, Varun Chandreskan, and Ben Simner for their feedback.

\bibliographystyle{abbrvnat}  
\bibliography{references}

\appendix
\appendix
\section*{Appendix}

\section{Diffie-Helman and One-Time Pad Encryption Attack}\label{sec:enc_attack}
For a concrete example of how a user could leverage the instruction following capabilities of an LLM to bypass output censorship, we introduce two classic cryptographic protocols, Diffie-Hellman-Merkle \cite{1055638} (DHM) key exchange and Stream Ciphers \cite{2003-ferguson-practical}. 
\par 
The DHM key exchange is a protocol for securely exchanging keys over an insecure communication channel, allowing two parties to agree on a shared secret key without exchanging secret information publicly. The DHM key exchange can be used to establish a shared secret key between a LLM user and the LLM in order to allow for secure communication without detection by censorship mechanisms. 
\par 
The protocol itself involves two parties $A$ and $B$ publicly agreeing to a large prime $p$ and small primitive root $g$ modulo $p$. Each party randomly generates a secret exponent $a$ and $b$ from $\{1,2,\dots,p-1\}$, communicating $g^a \mod p$ and $g^b \mod p$ with each other publicly. Then, party $A$ computes $(g^{b}\mod p)^a = g^{ba} \mod p$ and party $B$ computes $(g^{a}\mod p)^b = g^{ab}\mod p$, both of which are equal due to the commutativity of exponentiation under the cyclic group. The security of this protocol relies on the computational hardness of the discrete logarithm problem that is necessary in order to deduce $a$ or $b$ given $g^a \mod p$ or $g^b \mod p$ when $p$ is a large prime. 
\par 
Once a secret key has been established it can be used as part of a Stream Cipher, an encryption algorithm which combines string characters with a secret key through XOR operations. This can be achieved by using a codebook $C$ such as Unicode which assigns a unique binary representation for each character. The security of the Stream Cipher depends on the length and randomness of the secret key used. When the secret key is at least as long as the message being encrypted and is uniformly distributed in the set of possible keys, one-time pad encryption, which is impossible to crack, is recovered as a special case of the Stream Cipher.
\par 
While the one-time pad encryption cannot be cracked even by brute force, it is possible that the secret key is communicated to the model through a backdoor, \eg~by poisoning the training data. \citet{carlini2023quantifying} showed that LLMs, particularly larger ones, can directly memorise training data with the larger models being more vulnerable to memorisation. Canary strings, strings consisting of random tokens, were included within many benchmarks such as Big-Bench \citep{srivastava2022imitation} to test if LLMs had been trained on the benchmarks with the assumption that if they were, these canary strings would have been memorized. Similarly, one can presume that by poisoning the massive web-scraped training datasets of LLMs with a secret key~\eg~with domain hijacking \citep{carlini2023poisoning}, the LLM may memorise the secret key which enables one to instruct it to perform one-time pad encryption using the memorised key.

Through such an attack, a user can ensure that the model will produce encrypted outputs without knowledge of either the impermissible strings that the model would generate nor permissible target strings. Moreover, the generated outputs furthermore have cryptographic guarantees which ensure that decrypting them and thereby determining permissibility is very hard, if not impossible.

\section{Prompt Templates for Syntactic Censorship}
\label{sec:prompt_templates}

In this section, we explore a method employing what we refer to as \textit{syntactic} censorship, which does not have the invariance property (\Cref{def:invar}) of semantic censorship. Syntactic methods involve determining the set of permissible strings in terms of qualities of the string itself rather than its content. For example, an output string could be deemed impermissible only if it contains profanity; string transformations may eliminate the profanity making the transformed output permissible.

However, such methods can also be unreliable~\cite{boucher2022bad,shumailov2021sponge,boucher2023vision}. For example, simple filters on what types of words are allowed can be bypassed by misspellings of those words. Often, misspelled words still carry the same harmful or dangerous meaning. Thus rather than defining such filters, we instead restrict all permissible outputs to a relatively small predefined set---any string that isn't a member of that set is impermissible.  

For effective syntactic censorship, we explore Prompt Templates, a method that restricts the set of permissible strings to relatively small predefined collections of permissible templates, consisting of strings and variable tokens. This reduces the problem of censorship to a classification problem over a possible prompt templates. For example, an LLM classifier could be employed to choose the most appropriate prompt template given an input string. This is a strong form of syntactic censorship: rather than imposing restrictions on what a string can or cannot contain, we impose restrictions on what a string can or cannot be, such that many perfectly safe strings are still deemed impermissible.

\subsection{Definition of Prompt Templates}
We formally define Prompt Templates as prompts containing variable name tokens that function as pointers to content inferred or generated from the original uncensored string which is stored in external memory. 
\begin{definition}[Prompt Template]
    A Prompt Template $T = (t_1, t_2, ..., t_n)$ is comprised of a sequence of $n$ tokens $t_i$. Each token $t_i \in \Sigma \cup V$ can be either a string token in vocabulary $\Sigma$ or a variable token denoted as $v_i$ in variable vocabulary $V$.
\end{definition}
The use of variable tokens to represent uncensored user input is intended for usage within a larger framework presented in \cref{sec:Secure Framework} where LLMs can interact with other LLMs or tools. In such settings one may want to allow other models or tools to have access to uncensored data, particularly if the output of those models is censored. We discuss the application of Prompt Templates within this context in greater detail in \cref{sec:Secure Framework}.

\subsection{Prompt Template Examples}
To illustrate what a Prompt Template mechanism would look like, we consider the following case of a user interacting with an LLM email Assistant with access to an email API. We provide examples of what Prompt Templates for user requests to the model could be.
\myexample{Example: User Request Prompt Templates}{
\begin{itemize}
\item Request to Draft an Email: "Help me draft an email to [Recipient] regarding [Subject]."
\item Request to Schedule a Meeting: "Please schedule a meeting with [Attendees] for [Date] at [Time]."
\item Request to Summarise an Email: "Provide a summary of the email from [Sender] with the subject [Subject]."
\item Request for Email Search: "Find all emails related to [Keyword/Topic] from [Sender/Recipient]."
\item Request for Follow-up Reminder: "Set a reminder to follow up with [Contact] regarding [Subject]."
\item Request for Calendar Availability: "Check my calendar for availability on [Date/Time]."
\item Request for Contact Information: "Provide contact information for [Contact/Company]."
\item Request for Email Forwarding: "Forward the email from [Sender] to [Recipient]."
\item Request for Unsubscribe Assistance: "Help me unsubscribe from [Mailing List/Newsletter]."
\item Request to Create an Email Signature: "Assist in creating a professional email signature for my account."
\item ...
\end{itemize}
}
For those tasks that involve external information provided by another individual, which could in turn pose security risks we consider another collection of possible prompt templates for summaries of emails
\myexample{Example: Email Summary Prompt Templates}{
\begin{itemize}
\item Meeting Request: [Sender's Name] requests a meeting on [Meeting Date] at [Meeting Time] for [Meeting Topic].
\item Project Update: [Sender's Name] shares project progress, including accomplishments, challenges, and next steps.
\item Action Required: [Sender's Name] assigns [Task/Action] with a deadline of [Due Date/Deadline].
\item Request for Information: [Sender's Name] requests [Information/Documentation] by [Deadline/Date].
\item Important Update: [Sender's Name] provides an important update regarding [Topic].
\item Meeting Follow-up: [Sender's Name] follows up on [Meeting/Conversation], outlining discussion points, action items, and assigned responsibilities.
\item Request for Review/Approval: [Sender's Name] requests a review/approval for [Document/Proposal/Request] by [Deadline/Date].
\item Thank You Note: [Sender's Name] expresses gratitude for [Event/Assistance/Support].
\item Invoice Payment Reminder: [Sender's Name] reminds about payment for Invoice [Invoice Number], amount due by [Due Date].
\item Subscription Renewal Notice: [Sender's Name] notifies about the upcoming renewal of [Subscription/Service Name] on [Renewal Date].
\item ...
\end{itemize}}
Through a large collection of Prompt Templates one can cover many of the possible tasks that would be requested of the LLM by a user or many of the possible outputs of the LLM when exposed to potentially risky information. Selection among a set of prompt templates can be performed by utilising another LLM to perform classification over the collection of templates. Moreover by decomposing the contents into separate variables, those variables could thereafter be parsed and classified by an LLM, or even directly inferred by using a software tool and thereby avoiding any risk of various prompt injections that would try to bypass the censorship method while improving utility.

\subsection{Practical Implications and Security Guarantees}
Due to the pre-selection of approved prompt templates $T$ by a model designer, this finite collection of possible options becomes small enough to allow exhaustive validation, regardless of the content associated with variable tokens. The pre-selection process can be flexible. Potential methods include analysing large amounts of user interactions with past LLMs for a given use case and clustering them to find prototype prompts that cover a wide spectrum of user interaction with models. 

These prototype strings can be exhaustively vetted and modified to ensure they satisfy any desired constraints and can be flexibly integrated with variable tokens. If the LLM outputs go to another LLM or external tool, any desired security guarantees can be determined through verifying effects of each of the prompt templates on downstream models. Thus, verifiable security of individual model outputs is provided through exhaustive verification.

Furthermore, we assert that the string transformation attack described in~\Cref{sec:Encryption} will be unlikely to successfully allow users to recover an impermissible output not already known to the user through transforming the prompt template output. In particular, assuming malicious users do not exactly know the impermissible output that the model would provide them with a-priori, the success of the encryption style attacks described in \cref{sec:Encryption} relies on the probability that the encrypted version of an impermissible output happens to match one of the permissible template strings. This occurrence that has an extremely low probability due to the relatively small number of the pre-selected prompt templates out of all possible output strings.  

Some of the decision making process that an LLM can engage in only requires the template without any of the values taken on by the variable tokens. For example which could instead be stored separately and passed onto future tools that would not be vulnerable to the same attacks that LLMs are. This method would be naturally and easily incorporated into the Dual-LLM defense \citep{willison_2023} that has been proposed, enabling additional security to ensure that the user does not provide impermissible inputs or receive impermissible content that may include social engineering attacks. This defence can be further generalised through the access control framework we describe in \cref{sec:Secure Framework}. 

By utilising input prompt templates, if the generative process by which outputs are created is set to be deterministic, then one can replace the LLM with a lookup table which maps each of the permissible input prompt templates to the corresponding LLM output. Doing so ensures that no vulnerabilities of LLMs themselves, including those that emerge due to their instruction following capabilities, could be leveraged by an attacker as an LLM no longer needs to be deployed in this setting. Nevertheless, such a lookup table would be far less useful than extant LLMs due to the huge restrictions on possible inputs.

\section{A Framework for the design of Secure LLM-Based systems} \label{sec:Secure Framework}

In this section, we explore a potential security inspired approach for managing risks by designing secure LLM-based systems and highlight the role that censorship can play in making these systems useful. As LLMs become integrated within larger systems and frameworks, acquiring access to tools and datasets, it is important to recognise the potential safety and security risks that arise and to equip model providers with the tools necessary to mitigate such risks.

We describe a framework for the design of such LLM based systems which extend the simple interactions between a user and an LLM to incorporate settings of multiple users, models, tools, and data sources. To ensure security, the proposed framework relies on access controls~\citep{anderson2020access} which separate all components into subject, objects, privileges and transformations with censorship playing the role of facilitating flow of information with transformations in the privileged group. We further demonstrate the role of censorship mechanisms as part of the framework, facilitating the flow of information while maintaining certain security guarantees. %

We leverage the frameworks introduced by classic access control models such as the Bell-LaPadula (BLP) model \citep{Bell1973SecureCS} and the Biba model \citep{biba1977integrity} and extend them to the setting of LLM-based systems which incorporate \texttt{tools} as well. We identify key security criteria, identify all entities that play a role within the system, and define the actions they can perform. To ensure security we define security levels, a hierarchy of degrees of trust in entities, and security properties which jointly determine how information can flow within the system so as to ensure security criteria are always satisfied. 

The restrictive nature of this access control framework limits the practical use of such an access control framework on its own, due to the strong restrictions on the flow of information between entities. However, this limitation elucidates the key role that verifiable censorship mechanisms, such as Prompt Templates, can play within such a framework, namely, they enable flow of information that is otherwise not permitted by the security constraints while still ensuring that desired security criteria are not violated and the system cannot end up in an unsecure state. This enables us to clearly formulate the utility and purpose of censorship within the broader context of designing secure LLM-based systems alongside making evident the significance of censorship limitations in being able to ensure security criteria are satisfied.

We first provide a definition of an LLM-based system before describing the framework for secure LLM-based systems in detail. 
\begin{definition}[\textbf{LLM-Based Model}]
We define an LLM-Based Model $F(M)$ to consist of a collection of LLM models $M := \{m_1,m_2,\dots,m_n\}$ that take strings as input and produce strings as output.
\end{definition}
Another integral component for the design of secure LLM-based systems are the security criteria. When considering LLM-based systems with tool access where said \texttt{tools} could have external consequences, it is natural to require that such \texttt{tools} are not misused. Thus, one security criteria is to prevent unauthorised tool usage. 

Another desirable criteria is to ensure integrity of information, that is to prevent unauthorised modification or generation of information as~\eg with Clark-Wilson model~\citep{6234890}. Alternatively, in some cases one may seek to ensure confidentiality of information as~\eg with BLP model~\citep{Bell1973SecureCS}, that is to prevent unauthorised access to information. We focus on information integrity as a more practical security concern for most settings in particular when concerns of prompt injection are involved. 

Properly establishing a framework of secure LLM-based systems requires identifying the complete flow of information within the model as well as the external entities that interact with it. In particular, besides the models we identify 
\begin{itemize}
    \item \textbf{(\texttt{subjects}):} The set of \texttt{subjects} $S := M \cup U$ where users $U := \{u_1, u_2,\dots, u_k\}$ are external entities who provide inputs such as prompts to the LLM-based model. %
    \item \textbf{(\texttt{tools}):} The set of \texttt{tools} $T := M \cup \{t_1,t_2,\dots,t_j\}$ are \texttt{tools} that can be utilised by models, including models themselves. These can include a code interpreter or an API. 
    \item \textbf{(\texttt{objects}):} The set of \texttt{objects} $O := \{o_1,o_2,\dots,o_l\}$ are files which \texttt{subjects} and \texttt{tools} can access and modify. 
\end{itemize}
Note that LLMs can function as \texttt{subjects} and \texttt{tools} as they can both initiate actions through instructions to \texttt{objects} or \texttt{tools} and execute instructions provided to them. One unique challenge of managing LLM-based systems is that LLMs cannot effectively distinguish inputs from \texttt{objects} and inputs from \texttt{subjects}. That is, if an LLM is granted access to an object which provides the LLM with data as input and that data contains instructions in the form of text, the LLM can treat that data as an input or prompt from another subject. This challenge is a major component of the prompt injection vulnerability of LLMs \citep{greshake2023youve}.

Having identified the entities participating within the system, we define the permissions, or possible actions, they are endowed with as these define the possible sources of risk and need to be managed carefully.

A subject has the following permissions:
\begin{itemize}
    \item \textbf{(Prompt:)} A subject can Prompt an \texttt{object} or \texttt{tool}. Prompting an \texttt{object} consists of requesting access to certain data, whereas prompting a \texttt{tool} consists of providing instructions for a \texttt{tool} to execute.
    \item \textbf{(Update:)} A subject can update the data stored within an object.
    \item \textbf{(Create:)} A subject can create \texttt{tools}, \texttt{objects}, and \texttt{subjects}.
\end{itemize}
Prompting a tool often implies expectation of an output. To each tool we assign an output object which stores its outputs, and thus prompting a tool can often involve both prompting the tool as well as prompting its output object.
A tool has the following permissions
\begin{itemize}
    \item \textbf{(Execute:)} A \texttt{tool} can execute instructions it has received from a \texttt{subject}
    \item \textbf{(Access:)} A \texttt{tool} can access data stored within an \texttt{object}
    \item \textbf{(Update:)} A \texttt{tool} can update the data stored within an \texttt{object}. This may be with newly generated data.
\end{itemize}
We distinguish between prompting and access, as \texttt{tools} are assumed to not conflate data received as input with instructions received as input, but \texttt{subjects}, namely LLMs, are not necessarily capable of distinguishing the two. Within our framework, the challenge of LLMs conflating data and instructions as input is due to their dual role as both subject and tool. 

To capture the notion of authorised and unauthorised actions, we define a totally ordered set of security levels $\mathcal{L}$ comparable by $\leq$, with each \texttt{subject} assigned to a security level (its clearance) and each object and tool assigned to a security level (their classification). These security levels impose restrictions on the access abilities of various \texttt{subjects} within the system. Models are assigned a single security level which determines both their clearance level as a \texttt{subject} and classification as a \texttt{tool}. However, \texttt{tools} and their output \texttt{objects} can be assigned different security levels, \eg~output \texttt{objects} can be assigned a security level lower than that of their corresponding tool.

To ensure that unauthorised tool usage does not occur, namely to prevent unauthorised execution of instructions provided to a tool, we define the following security property
\begin{definition}[\textbf{No Unauthorised Usage}]\label{def:usage}
    \texttt{Tools} are only permitted to execute instructions received from \texttt{subjects} of the same security level or higher. 
\end{definition}
Next, to ensure integrity of information we define the following properties
\begin{definition}[\textbf{Simple Security Property}]\label{def:simple}
    A \texttt{subject} at a given security level is only permitted to prompt \texttt{objects} at the same security level or higher, and prompt \texttt{tools} at the same security level. A \texttt{tool} at a given security level is only permitted to access data from \texttt{objects} at the same security level or higher
\end{definition}
Furthermore, we define the 
\begin{definition}[\textbf{* Security Property}]\label{def:star}
    A \texttt{subject} or \texttt{tool} at a given security level are only permitted to update \texttt{objects} of the same security level or lower.
\end{definition}
While users will be assigned fixed security levels by model designers, many models and \texttt{tools} will be assigned to the lowest security level initially and will inherit the security level of the subject which first prompts them.

Finally we define the Creation Security Property which regulates the creation of new \texttt{objects} and \texttt{tools} (\eg~calendars)
\begin{definition}[\textbf{Creation Security Property}]\label{def:creation}
    A \texttt{subject} at a given security level is only permitted to create \texttt{objects} and \texttt{tools} of the same security level or lower
\end{definition}

All these properties are defined within the context of having a security criterion of ensuring integrity of information and mitigation of unauthorised tool usage, however, they can be easily adapted for other criteria such as ensuring confidentiality of information, wherein \texttt{subjects} would only be able to prompt \texttt{objects} or \texttt{tools} at the same security level or higher for example. Thus, this makes for a general access control framework that can be easily adjusted to various security criteria through modifications of the security properties. 

Our formulation can be reduced to classical models such as the BLP or Biba model, allowing for the same security guarantees to apply. In particular, by treating tools as subjects and endowing subjects with the permission of accessing rather than prompting, the properties \Cref{def:simple} and \Cref{def:star} are reduced to the standard properties for the Biba model. Furthermore, as our model does not allow subjects to request access to change security levels at all,  \Cref{def:creation} becomes equivalent to the Invocation property for the Biba model and consequently we can conclude that information only flows downward within the model and integrity is maintained. Given this, we are also able to ensure that no unauthorised usage of ``tools'' occurs as any instruction for execution cannot have originated from an entity at a lower security level.  

Exceptions to the aforementioned security properties can be made through input or output censorship. An exception to \Cref{def:usage} can be made if the input to a tool is censored such that it does not pose any security risks of misuse. As censorship of these inputs requires censorship of instructions to execute, the Undecidability challenges apply which makes determining whether any provided instruction poses a security risk difficult. An exception to \Cref{def:simple} can be handled by input or output censorship depending on the context, and an exception to \Cref{def:star} can be handled by output censorship. 

The proposed Prompt Template censorship maps inputs and outputs to elements of an approved whitelist set, thus preventing tool misuse, \texttt{objects} from receiving corrupted information from lower ranked \texttt{subjects} or the corruption of higher ranked \texttt{tools} or \texttt{objects} by prompts from lower ranked \texttt{subjects} or \texttt{objects}.

\end{document}